\documentclass[11pt]{article}
\pdfoutput=1
\usepackage{enumerate}
\usepackage{pdfsync}
\usepackage[OT1]{fontenc}
\usepackage{natbib}
\usepackage[usenames]{color}
\usepackage{uclaml}
\usepackage[colorlinks,
            linkcolor=red,
            anchorcolor=blue,
            citecolor=blue
            ]{hyperref}
\usepackage{fullpage}
\usepackage{hyperref}
\usepackage[protrusion=true,expansion=true]{microtype}
\usepackage{pbox}
\usepackage{setspace}
\usepackage{tabularx}
\usepackage{float}
\usepackage{wrapfig,lipsum}
\usepackage{enumitem}
\usepackage{microtype}
\usepackage{graphicx}
\usepackage{subfigure}
\usepackage{enumerate}
\usepackage{enumitem}
\usepackage{pgfplots}
\usetikzlibrary{arrows,shapes,snakes,automata,backgrounds,petri}
\usepackage{booktabs} 

\allowdisplaybreaks

\usepackage[utf8]{inputenc} 
\usepackage[T1]{fontenc}    
\usepackage{hyperref}       
\usepackage{url}            
\usepackage{booktabs}       
\usepackage{amsfonts}       
\usepackage{nicefrac}       
\usepackage{microtype}      
\usepackage{xcolor}         

\newcommand{\qvalue}{Q}
 \newcommand{\vvalue}{V}
 
 \newcommand{\reward}{r}

\usepackage{wrapfig}

\usepackage{pifont}

\def \alg {\mathtt{Alg}}

\usepackage{uclaml}

\newcommand{\la}{\langle}
\newcommand{\ra}{\rangle}

\def \algname {\text{$\cF$-UPAC-OFUL }}

\usepackage{enumitem}
\usepackage{algorithm}
\usepackage{algpseudocode}
\allowdisplaybreaks

\ifdefined\final
\usepackage[disable]{todonotes}
\else
\usepackage[textsize=tiny]{todonotes}
\fi
\usepackage{colortbl}
\definecolor{LightCyan}{rgb}{0.8, 0.9, 1}

\makeatletter
\newcommand*{\rom}[1]{\expandafter\@slowromancap\romannumeral #1@}
\makeatother
\title{\huge Uniform-PAC Guarantees for Model-Based RL with Bounded Eluder Dimension}

\author
{
    Yue Wu\thanks{Department of Computer Science, University of California, Los Angeles, CA 90095, USA; e-mail: {\tt ywu@cs.ucla.edu}} 
    ~~~and~~~
	Jiafan He\thanks{Department of Computer Science, University of California, Los Angeles, CA 90095, USA; e-mail: {\tt jiafanhe19@ucla.edu}} 
	~~~and~~~
	Quanquan Gu\thanks{Department of Computer Science, University of California, Los Angeles, CA 90095, USA; e-mail: {\tt qgu@cs.ucla.edu}}
}

\date{}

\begin{document}
\maketitle

\begin{abstract}
    Recently, there has been remarkable progress in reinforcement learning (RL) with general function approximation. However, all these works only provide regret or sample complexity guarantees. It is still an open question if one can achieve stronger performance guarantees, i.e., the uniform probably approximate correctness (Uniform-PAC) guarantee that can imply both a sub-linear regret bound and a polynomial sample complexity for any target learning accuracy.  We study this problem by proposing algorithms for both nonlinear bandits and model-based episodic RL using the general function class with a bounded eluder dimension. The key idea of the proposed algorithms is to assign each action to different levels according to its width with respect to the confidence set. The achieved uniform-PAC sample complexity is tight in the sense that it matches the state-of-the-art regret bounds or sample complexity guarantees when reduced to the linear case. To the best of our knowledge, this is the first work for uniform-PAC guarantees on bandit and RL that goes beyond linear cases.
\end{abstract}

\section{Introduction}

Designing efficient algorithms to learn and plan in the sequential decision-making environment modeled by a Markov decision process (MDP) is one of the main tasks in reinforcement learning (RL). However, traditional tabular RL algorithms suffer from the curse-of-dimensionality due to the large size of the state and action spaces in practice. To enable learning in high-dimensional state and action spaces, using a predefined function class to approximate the underlying transition dynamic or the value function is a common approach. Most existing works for RL with function approximation focus on simple linear function classes such as the linear mixture MDP \citep{modi2020sample,ayoub2020model,zhou2020provably}, which can replace the size of the state and action spaces with the dimension of the linear function class. However, these assumptions are often too restrictive to hold in practice. Recently, a line of works \citep{russo2013eluder,du2021bilinear,jin2021bellman} emerged that studies RL with general function approximation, introducing new complexity measures for the general function class and proposing new algorithms with regret bounds or PAC guarantees in terms of the complexity of the general function class.
All existing results of RL with a general function class are limited to either regret bounds or PAC sample complexity, both of which cannot ensure convergence to the optimal policy up to arbitrary accuracy. To address this, \citet{dann2017unifying} proposed the strongest possible performance measure, the uniform-PAC guarantee, which provides upper bounds on the number of $\epsilon$-suboptimal episodes for any accuracy parameter $\epsilon>0$ uniformly. They also proposed the UBEV algorithm with $\tilde O(SAH^4/\epsilon^2)$ Uniform-PAC sample complexity, which guarantees convergence to the optimal policy for any $\epsilon>0$. Recently, \citet{he2021uniform} proposed the FLUTE algorithm for RL with linear function approximation, which has an $\tilde O(d^3H^5/\epsilon^2)$ uniform-PAC sample complexity, where $d$ is the dimension of the linear function class.

In this paper, we propose new algorithms for both the bandit problem and model-based RL problem with a general function class, focusing on a model-based RL problem called linear mixture MDPs \citep{ayoub2020model}. Our approach uses the eluder dimension \citet{russo2013eluder} as a complexity measure for the general function class $\mathcal{F}$, which generalizes the linear independence relation in the standard vector space to capture the nonlinear independence in the function space approximately. The core of our algorithms is to assign each action to different levels based on its width with respect to the confidence set. For function classes with a bounded eluder dimension, we prove that both algorithms have uniform-PAC guarantees and near-optimal sample complexity bounds. Our key contributions are highlighted below. 
\begin{itemize}[leftmargin = *] 
 \item In the nonlinear bandit problem, where the reward function $f$ is from a known class $\mathcal{F}$ with bounded eluder dimension, we propose the $\cF$-UPAC-OFUL algorithm that achieves an $\tilde O(d_K d_E/\epsilon^2)$ uniform-PAC guarantee. Here $d_K$ relates to the metric entropy (log-covering number) of $\mathcal{F}$, $d_E$ is its eluder dimension, and $\epsilon$ is the accuracy parameter. Our result implies a $\tilde O(\sqrt{d_K d_E K})$ regret guarantee in the first $T$ rounds, matching the result in \citep{russo2013eluder} up to a logarithmic factor. This is the first uniform-PAC guarantee for nonlinear bandits.
\item We also consider a model-based episodic RL problem where the transition probability $P^*$ belongs to a known family $\cP$. We propose the $\cF$-UPAC-VTR algorithm and obtain a $\tilde O(H^3 d_K d_E/\epsilon^2)$ uniform-PAC sample complexity, where $H$ is the horizon length, $d_K$ relates to the metric entropy and $d_E$ is the eluder dimension for the value function class $\mathcal{F}$ induced by the transition probability family $\cP$. This result also implies a $\tilde O(\sqrt{d_K d_E H^3 K})$ regret bound in the first $K$ episodes, matching the result of the UCRL-VTR algorithm in \citep{ayoub2020model} up to a logarithmic factor.

\end{itemize}

For the ease of comparison, we list the results of our algorithms and the most related algorithms in Tables~\ref{table:bandit} and~\ref{table:mdp}.

\newcolumntype{g}{>{\columncolor{LightCyan}}c}
\begin{table}[ht]
\caption{Comparison of algorithms for bandits with linear or general reward function in terms of sample complexity and regret. Note that $d = \tilde{\Theta}(d_K) = \tilde{\Theta}(d_{E})$ in the linear case. ``--" means the corresponding result is not available or not applicable for uniform-PAC sample complexity. Our results are tight given the nearly matching upper and lower bounds. }\label{table:bandit}
\centering
\resizebox{\textwidth}{!}{
\begin{tabular}{cgggg}
\toprule
\rowcolor{white}
 Algorithm & Sample Complexity & Regret & Uniform-PAC & Reward Function
 \\
\midrule
\rowcolor{white}
OFUL & & &  & \\

\rowcolor{white} \small{\citep{abbasi2011improved}} & \multirow{-2}{*}{--}  &  \multirow{-2}{*}{$\tilde  O(d \sqrt{K})$} & \multirow{-2}{*}{\ding{56}}   & \multirow{-2}{*}{Linear }\\

\rowcolor{white}
UPAC-OFUL & & & & \\

\rowcolor{white} \small{\citep{he2021uniform}} & \multirow{-2}{*}{$\tilde  O(d^2/\epsilon^2)$} &  \multirow{-2}{*}{$\tilde  O(d \sqrt{K})$} & \multirow{-2}{*}{\ding{51}}   & \multirow{-2}{*}{Linear }\\
 
\rowcolor{white}
Linear UCB & & & & \\

\rowcolor{white} \small{\citep{russo2013eluder}} & \multirow{-2}{*}{--}  &  \multirow{-2}{*}{$\tilde  O(\sqrt{d_K d_E K})$} & \multirow{-2}{*}{\ding{56}}  & \multirow{-2}{*}{General }\\

\algname & & & & \\

 \small{Our work} & \multirow{-2}{*}{$\tilde  O(d_K d_E/\epsilon^2)$}  &  \multirow{-2}{*}{$\tilde  O(\sqrt{d_K d_E K})$} &  \multirow{-2}{*}{\ding{51}}  & \multirow{-2}{*}{General }\\

 \midrule
 \rowcolor{white}
 Lower bound & & & & \\
\rowcolor{white}
 \small{\citep{lattimore2018bandit}} & \multirow{-2}{*}{--}  &  \multirow{-2}{*}{$\tilde  \Omega(d\sqrt{K})$}  & \multirow{-2}{*}{--}  & \multirow{-2}{*}{Linear }\\
\bottomrule
\end{tabular}
}
\end{table}

\begin{table*}[ht]
\caption{Comparison of algorithms for model-based RL with linear or general function approximation in terms of sample complexity and regret. Note that $d = \tilde{\Theta}(d_K) = \tilde{\Theta}(d_{E})$ in the linear case. }\label{table:mdp}
\centering
\resizebox{\textwidth}{!}{
\begin{tabular}{cgggg}
\toprule
\rowcolor{white}
 Algorithm & Sample Complexity & Regret & Uniform-PAC & Function Approximation
 \\
\midrule
\rowcolor{white}
UCRL-VTR+ & & & & \\

\rowcolor{white} \small{\citep{zhou2020nearly}} & \multirow{-2}{*}{--}  &  \multirow{-2}{*}{$\tilde  O(dH \sqrt{K})$} &  \multirow{-2}{*}{\ding{56}}  & \multirow{-2}{*}{Linear }\\

\rowcolor{white}
UCRL-VTR & & & & \\

\rowcolor{white} \small{\citep{ayoub2020model}} & \multirow{-2}{*}{--}  &  \multirow{-2}{*}{$\tilde  O(\sqrt{H^3d_K d_E K})$} &  \multirow{-2}{*}{\ding{56}}  & \multirow{-2}{*}{General }\\

$\mathcal{F}$-UCRL-VTR & & & & \\

 \small{Our work} & \multirow{-2}{*}{$\tilde  O(H^3d_K d_E/\epsilon^2)$}  &  \multirow{-2}{*}{$\tilde  O(\sqrt{H^3d_K d_E K})$} &  \multirow{-2}{*}{\ding{51}}  & \multirow{-2}{*}{General }\\

 \midrule
 \rowcolor{white}
 Lower bound & & & & \\
\rowcolor{white}
 \small{\citep{zhou2020nearly}} & \multirow{-2}{*}{--}  &  \multirow{-2}{*}{$\tilde  \Omega(dH\sqrt{K})$} &  \multirow{-2}{*}{--}  & \multirow{-2}{*}{Linear }\\
\bottomrule
\end{tabular}
}
\end{table*}

\section{Related Work}
\subsection{RL with Linear Function Approximation}

There is a body of research on learning Markov Decision Processes (MDPs) under the assumption of linear function approximation, which can be divided into model-free and model-based approaches. For model-free algorithms, \citet{jin2019provably} studied the linear MDP model, where the reward function $r(s,a)$ and transition probability function $\PP(s'|s,a)$ are linear with respect to a given feature map $\bphi(s,a)$. The LSVI-UCB algorithm proposed by \citet{jin2019provably} achieved an $O(\sqrt{d^3H^3K})$ regret bound. \citet{zanette2020learning} extended the linear MDP assumption to the low inherent Bellman error assumption, where the Bellman backup can be approximated by a linear function class.

For model-based algorithms, \citet{jia2020model} investigated linear mixture MDPs \citep{modi2020sample}, where the transition probability function $\PP(s'|s,a)$ is linear with respect to a given feature mapping $\bphi(s'|s,a)$. They proposed the UCRL-VTR algorithm, which has a regret guarantee of $\tilde O(d\sqrt{H^3K})$. \citet{zhou2020nearly} improved the regret guarantee to $\tilde O(dH\sqrt{K})$ by introducing a Bernstein-type bonus in the linear mixture model.

Linear bandit problems, as a special case of RL with linear function approximation, have also been extensively studied. For infinite-arm linear bandit problems, \citet{dani2008stochastic} obtained an $O(d\sqrt{K\log^3 K})$ regret guarantee with the Confidence Ball algorithm. \citet{abbasi2011improved} proposed the OFUL algorithm, which improved the result to $O(d\sqrt{K\log^2 K})$.


\subsection{RL with General Function Approximation}
Recently, a line of research has emerged in RL with general function approximation to relax the linear function assumption \citep{russo2013eluder,jiang2017contextual,sun2019model,dong2020root,yang2020function,wang2020reinforcement,ayoub2020model,jin2021bellman,du2021bilinear}. \citet{jiang2017contextual} introduced the Bellman rank, a complexity measure that depends on the function class $\mathcal{F}$ and the roll-in policy, and proposed the OLIVE algorithm for models with low Bellman rank, which has a polynomial PAC-bound guarantee. The AVE algorithm by \citet{dong2020root} was proposed for learning low-Bellman-rank MDPs and obtained the first $O(\sqrt{K})$-regret guarantee. \citet{sun2019model} extended the Bellman rank to the model-based setting and proposed a complexity measure called "witness rank", which is always no larger than the Bellman rank and obtained a polynomial PAC bound in terms of the "witness rank".

In contrast, the eluder dimension \citep{russo2013eluder} measures the complexity of a function class $\mathcal{F}$ from a different perspective, indicating how effectively the underlying function $f\in \mathcal{F}$ can be deduced from the current information. UCB-type and Thompson sampling-type algorithms for bandit problems were proposed by \citet{russo2013eluder}, both obtaining a $\sqrt{K}$-regret guarantee for bandit models with low eluder dimension. \citet{osband2014model} extended this notion to reinforcement learning models and proposed a post-sampling algorithm, while \citet{wang2020reinforcement} and \citet{ayoub2020model} extended it to more general model-free and model-based reinforcement learning problems. Under the assumption that the value function or the transition probability belongs to a function class $\mathcal{F}$ with low eluder dimension, both the model-free algorithm $\mathcal{F}$-LSVI \citep{wang2020reinforcement} and the model-based algorithm UCRL-VTR achieve $O(\sqrt{K})$-regret guarantees.

Recently, \citet{jin2021bellman} extended the eluder dimension to the Bellman eluder dimension, which considers the general function class and possible MDP structures simultaneously. The low Bellman eluder dimension covers both the low eluder dimension and the low Bellman rank, and the GOLF algorithm by \citet{jin2021bellman} achieves both $O(\sqrt{K})$-regret and polynomial PAC-bound guarantees. However, the Bellman eluder dimension does not cover the model-based RL problem. Concurrently, \citet{du2021bilinear} proposed the Bilinear class with bounded effective dimension, which contains many MDP models such as low Bellman rank and low witness rank, and the BiLin-UCB algorithm with polynomial PAC bound. Notably, neither the Bilinear class nor the low Bellman eluder dimension class contains each other.

All these works with general function approximation concern either a PAC sample complexity bound or a regret bound. This motivates us to provide a stronger guarantee on general function approximation, and we begin with the low eluder dimension as a first attempt.

\section{Preliminaries} \label{sec:prelim}
\paragraph{Notation} We use lower case letters to denote scalars,  lower and upper case bold letters to denote vectors and matrices. We use $\| \cdot \|_{\infty}$ to indicate the infinity norm. 
For a probability measure $P(s)$ and a function $V(s)$ on the space $\cS$, we use $\la P, V \ra$ or $\la P(\cdot), V(\cdot) \ra$ to denote the integral $\int_{\cS} V(s) dP(s)$.
We also use the standard $O$ and $\Omega$ notations. We say $a_n = O(b_n)$ if and only if $\exists C > 0, N > 0, \forall n > N, a_n \le C b_n$; $a_n = \Omega(b_n)$ if $a_n \ge C b_n$. The notation $\tilde{O}$ is used to hide logarithmic factors. 

\subsection{Episodic Markov Decision Process} \label{subsec:prelim-MDP} 
In this work, we consider the homogeneous, episodic Markov Decision Process (MDP). Each episodic MDP is denoted by a tuple $M\big(\cS, \cA, H, r(
\cdot,
\cdot), P^*(\cdot|\cdot,
\cdot)\big)$. Here, $\cS$ is the state space, $\cA$ is the finite action space,  $H$ is the horizon length (i.e., length of each episode), $r: \cS \times \cA \rightarrow [0,1]$ is the reward function and $P^*(s'|s,a): \cS \times \cA \rightarrow \Delta^{|\cS|}$ is the transition probability function that denotes the probability for state $s$ to transit to the next state $s'$ given action $a$. 
A policy $\pi_h(\cdot|s) : \cS \times [H] \rightarrow \Delta^{\cA}$ is a function which maps a state $s$ and the current step $h$ to a distribution over the action space $\cA$. In addition, for any policy $\pi$ and step $h\in [H]$, we define the action-value function $\qvalue_h^{\pi}(s,a)$ and value function $\vvalue_h^{\pi}(s)$ as the expected total reward from step $h$ by following the policy $\pi$:
\begin{align*}
    V^{\pi}_h(s)
     &:=
    \EE_{\pi}
    \big[ 
    {\sum_{h'=h}^{H}}
    r(s_{h'}, a_{h'})
    | s_h = s
    \big], \\
    Q^{\pi}_h(s,a) 
     &:= 
    r(s,a) 
    + \EE_{\pi} \big[{\sum_{h'=h+1}^H} r(s_{h'}, a_{h'})\big| s_h=s,a_h=a\big],
\end{align*}
where $s_{h'+1}\sim P^*(\cdot|s_{h'},a_{h'})$ and $a_{h'} \sim \pi_h(s_{h'})$. With this definition, the value function $\vvalue_h^{\pi}(s)$ and  $\qvalue_h^{\pi}(s,a)$ are bounded in $[0,H]$.

We define the optimal value function $V_h^*$ and $Q_h^*$ as $V_h^*(s) = \sup_{\pi}\vvalue_h^{\pi}(s)$ and $\qvalue_h^*(s,a) = \sup_{\pi}\qvalue_h^{\pi}(s,a)$.  For each step $h\in[H]$ and policy $\pi$, we have the following Bellman equation and Bellman optimality equation:
\begin{align}
    \qvalue_h^{\pi}(s,a) & = \reward(s,a) + \la P^*(\cdot |s,a),  \vvalue_{h+1}^{\pi}(\cdot) \ra , \notag \\
    \qvalue_h^{*}(s,a)  & = \reward(s,a) + \la P^*(\cdot |s,a), \vvalue_{h+1}^{*}(\cdot) \ra,\label{eq:bellman}
\end{align}
where $\vvalue^{\pi}_{H+1}(s')=\vvalue^{*}_{H+1}(s')=0$.

We study the online RL problem where the learning agent is given with $s\in \cS, a\in \cA, h\in [H]$ and reward $r$ but does not know the transition probability $P^* \in \cP$, where $\cP$ is a class of possible transition probabilities. 

Generally speaking, the goal of an RL agent is to maximize the expected total reward over all $K$ episodes. When $P$ is known, the optimal policy is also known and computable via dynamic programming. 
Denote $\pi_k$ as the policy the agent follows at episode $k$, the suboptimality gap incurred at episode $k $ is defined as the difference between the optimal value function and value function for policy $\pi_k$:
    $\Delta_{k}
     :=
    V^*_{1}(s_{k,1})
    -
    V_1^{\pi_k}(s_{k,1})$.
With this notation, the pseudo-regret in first $K$ episodes is
\begin{align*}
{
    \text{Regret}(K)=\sum_{k=1}^K \Delta_{k}=\sum_{k=1}^K V^*_{1}(s_{k,1})
    -
    V_1^{\pi_k}(s_{k,1}).}
\end{align*}
Most works in the literature focus on establishing an upper bound on $\text{Regret}(K)$. In the next subsection, we introduce the notion of a stronger guarantee.

\subsection{The Uniform-PAC Guarantee}
We say an algorithm is $(\epsilon, \delta)$-PAC, if for any $\epsilon, \delta \in (0,1)$, there exists a function $N(\epsilon, \delta)$ that is polynomial in $\epsilon^{-1}$ and $\log(\delta^{-1})$, such that
\begin{align*}
{
    \PP\big(  \sum_{k=1}^\infty\ind \{\Delta_k
    >\epsilon \} \le  N(\epsilon, \delta)\big) \ge 1 - \delta.
}
\end{align*}
Here $N(\epsilon, \delta)$ is the sample complexity function.
However, both the regret guarantee and the PAC guarantee have their limitations. For example, algorithms with sub-linear regret $o(K)$ in the first $K$ episodes may suffer $\epsilon$-suboptimality infinite times and fail to learn the optimal policy. For the PAC bound, it only controls the number of times that $\Delta_k>\epsilon$. And the algorithm may incur a smaller $\epsilon'$-suboptimality gap for infinite times for $\epsilon' < \epsilon$. 

To overcome these limitations, \citet{dann2017unifying} introduced a stronger notion of guarantee called uniform-PAC, which provides PAC-guarantees for all accuracy parameter $\epsilon$ uniformly. More specifically, we say an algorithm is uniform-PAC for some $\delta \in (0,1)$, if there exists a function $N(\epsilon, \delta)$ polynomial in $\epsilon^{-1}$ and $\log(\delta^{-1})$, such that
\begin{align*}
{
    \PP\big(\forall \epsilon >0,\ \sum_{k=1}^\infty\ind\{\Delta_k
    >\epsilon \} \leq N(\epsilon, \delta)\big) \geq 1- \delta.
    }
\end{align*}
\citet{dann2017unifying} showed that the uniform-PAC guarantee is \textit{strictly stronger} than both the regret bound and the PAC-guarantee:
\begin{theorem}[Theorem 3,  \citet{dann2017unifying}]\label{theorem: pac-transfer}
If an algorithm $\alg$ is $N(\epsilon,\delta)$-uniform-PAC  with sample complexity $N(\epsilon,\delta)=\tilde O(C_1/\epsilon+ C_2/\epsilon^2)$, where $C_1,C_2$ are constant parameter for the algorithm $\alg$ and only depend on $\text{poly}\big(S,A,H,\log(1/\delta)\big)$. Then, the algorithm $\alg$ has the following results:
\begin{itemize}[leftmargin = *]
    \item 1: $\alg$ will converge to optimal policies with high probability at least $1-\delta$: $\PP\big(\lim_{k \rightarrow +\infty} \Delta_k=0\big)\ge 1-\delta$ 
    \item 2:  Algoroithm $\alg$ is also $(\epsilon,\delta)$-PAC with the same sample complexity $\tilde O(C_1/\epsilon+ C_2/\epsilon^2)$ for all $\epsilon>0$.
    \item 3: With probability at least $1-\delta$, for each $K\in \NN$, the regret for $\alg$ in the first $K$ episodes is upper bounded by $\tilde O\big(\sqrt{C_2K}+\max\{C_1,C_2\}\big)$.
\end{itemize}
\end{theorem}
Due to the strong implication of the uniform-PAC guarantee, one may wonder if uniform-PAC is also achievable under the setting of general function approximation. 



\subsection{Complexity Measure of a Function Class}
To deal with general function class $\cF$, we will use two complexity measures.
The first one is covering number, which is formally defined as follow.
\begin{definition}
Suppose $\| \cdot \|$ is a norm on $\cF$ and $\alpha > 0$. A $\alpha$-covering with respect to $\| \cdot \|$ is a subset $\cG \subseteq \cF$, such that $\forall f \in \cF, \exists g \in \cG, \text{s.t. } \| f - g \| \leq \alpha$. The covering number $\cN(\cF, \alpha, \| \cdot \|)$ is the minimal cardinality of any $\alpha$-covering of $\cF$ with respect to $\| \cdot \|$.
\end{definition}

For many function classes, the log-covering number (a.k.a. metric entropy) $\log \cN(\cF, \alpha, \| \cdot \|)$ is linear in the dimension of $\cF$ and only logarithmic in $\alpha^{-1}$. One such function class $\{ f_{\btheta} | \btheta \in \bTheta \}$ has been discussed in \citet{russo2013eluder}, where $f_{\btheta}(\xb)$ is $L$-Lipschitz with respect to $\btheta$ and $\bTheta = [0,1]^{d}$. It can be shown that $\log \cN(\cF, \alpha, \|\cdot\|_{\infty}) \leq d \log(1 + L/\alpha)$.
Therefore, we make the following mild assumption.
\begin{assumption} \label{assumption:covering}
The metric entropy $\log \cN(\cF, \alpha, \|\cdot\|_{\infty})$ of the function class $\cF$ is bounded linearly by $\log(\alpha^{-1})$, i.e.,
$
    \log \cN(\cF, \alpha, \|\cdot\|_{\infty})
     \le 
    d_K \log(\alpha^{-1})$.
\end{assumption}
We can view $d_K$ as an upper bound on the Kolmogorov dimension  of the function class $\cF$ (see e.g., \citet{osband2014model}for more details).



The other complexity measure is the eluder dimension, which is first proposed by \citet{russo2013eluder} based on the concept of $\epsilon$-independence:
\begin{definition}
An input $\xb \in \cX$ is $\epsilon$-dependent on inputs $\{ \xb_1, \xb_2, \dots, \xb_n \} \subseteq \cX$ with repect to $\cF$ if any pair of function $f_1, f_2 \in \cF$ satisfying $\sum_{i=1}^{n} \big( f_1(\xb_i) - f_2(\xb_i) \big)^2 \le \epsilon^2$ also satisfies $f_1(\xb) - f_2(\xb) \le \epsilon$. $\xb$ is $\epsilon$-independent of $\{ \xb_1, \xb_2, \dots, \xb_n \} \subseteq \cX$ with respect to $\cF$ if it is not $\epsilon$-dependent on $\{ \xb_1, \xb_2, \dots, \xb_n \} \subseteq \cX$. 
\end{definition}
Then the eluder dimension is formally defined as follows.
\begin{definition}
The \textit{$\epsilon$-eluder dimension} $\mathrm{dim}_E(\cF, \epsilon)$ is the length $d$ of the longest sequence of elements in $\cX$ such that, for some $\epsilon' > \epsilon$, every element is $\epsilon'$-independent of its predecessors.
\end{definition}

\section{ Uniform-PAC Bounds for Nonlinear Bandits}

\subsection{The Bandit Problem with General Reward Functions}
In the nonlinear bandit problem, at each round $k \in \NN$, the agent selects an action $\xb_k$ from the action set $\cA_k$, and then receives the reward $R_k = f_{\btheta^*}(\xb_k) + \eta_k$, where the true reward function $f_{\btheta^*}$ is assumed to lie in a set of bounded, real-valued functions $\cF = \{ f_{\btheta}: \cX \rightarrow [0,1] | \btheta \in \bTheta \}$ which is indexed by $\btheta\in \bTheta$. 
$\eta_k$ is a conditionally unbiased $1$-sub-Gaussian noise:
\begin{align*}
    \forall k \in \NN, \lambda \in \RR,
    \EE[e^{\lambda \eta_k} |  \xb_1, \eta_1, \xb_2, \dots, \xb_k] \le e^{\lambda^2 /2}.
\end{align*}

Formally speaking, an agent selecting $\{ \xb_k \}_{k \in \NN}$ achieves uniform-PAC with complexity $N(\epsilon, \delta)$ if and only if
\begin{align*}
{
    \PP\big(\forall \epsilon >0,\ \sum_{k=1}^\infty\ind\{\Delta_k
    >\epsilon\} \le  N(\epsilon, \delta)\big) \ge 1 - \delta,}
\end{align*}
where $\Delta_k: = \max_{\xb \in \cA_k} f_{\btheta^*}(\xb) -  f_{\btheta^*}(\xb_k)$ denotes the suboptimality gap at round $k$.

\subsection{Algorithm}

We first present an $\cF$-UPAC-OFUL algorithm in Algorithm~\ref{algorithm:bandit}. The high-level idea of Algorithm~\ref{algorithm:bandit} is to split the rounds into several disjoint sets $\cC^{l}$, and apply the approach of optimistic exploration within each set. 

More specifically, in each set $\cC^{l}$, we use the information from the rounds $k$ such that $k \in \cC^l$ to construct the confidence set $\cF^l \subseteq \cF$ (Line~\ref{algline:confidence-set}), and try to find the most optimistic action (Lines~\ref{algline:take-action-1}-\ref{algline:take-action-2}), namely the action with largest upper confidence bound (UCB).
Line~\ref{algline:confidence-set} defines the confidence set $\cF^{l}$ with:
\begin{align}
    \cL_{\cC^{l}}
    (f, \hat{f})
    & :=
    {\sum_{k \in \cC^{l}}}
    \big(
    f(\xb_k) - \hat{f}(\xb_k)
    \big)^2, \notag \\
    \hat{f}^l
    & :=
     {\arg \min_{f \in \cF}  
    \sum_{k \in \cC^{l}}
    \big(
    f(\xb_k) - R_k
    \big)^2,}
    \notag \\
    \beta^{l}_t
    & := 
     {8  \log
    \big(\cN(\cF, \alpha, \|\cdot\|_{\infty}) / \delta
    \big)} 
    +
     {2 \alpha t
    \big( 8  + \sqrt{8  \log(4t^2/  2^{-l} \delta)}
    \big)},  \label{eqn:beta-bandit}
\end{align}
where the metric entropy can be replaced by its upper bound.
When the action $\xb_k$ is chosen, the algorithm needs to decide which level set the index $k$ should be assigned to (Lines~\ref{algline:assign-level-1}-\ref{algline:assign-level-2}). The decision is made by the scale of the width $w_{\cF^{l}}(\xb_k)$ at Line~\ref{algline:find-level-1}:
\begin{align*}
    w_{\cF^{l}}(\xb_k)
    & :=
    {
    \sup_{f \in \cF^{l}} f(\xb_k)
    -
    \inf_{f \in \cF^{l}} f(\xb_k)}.
\end{align*}
Once the index $k$ is assigned to a particular set, the algorithm updates the total level $S$ as the number of non-empty sets.


It can be shown that when $\cF$ is a linear function class, i.e., $f_{\btheta^*}(\xb_k) = \la \btheta^*, \xb_k \ra $, the confidence set for level $l$ becomes an ellipsoid $\{ \btheta | \| \btheta - \hat{\btheta}^l \|_{\Vb^l} \le \beta^l_t \}$, where $\Vb^l = \sum_{k \in \cC^l} \xb_k \xb_k^{\top}$ is the covariance matrix of the contexts at level $l$. In this case, the width $w_{\cF^{l}}(\xb)$ has a closed form $w_{\cF^{l}}(\xb) = 2 \beta_T^l \| \xb \|_{\Vb^l}$. This is exactly the bonus term used in LinUCB/OFUL, and \citet{he2021uniform} uses this term as a criterion to assign the contexts to the appropriate level in order to achieve uniform PAC guarantee for UPAC-OFUL.

\begin{algorithm}[H]
\caption{\algname}\label{algorithm:bandit}
\begin{algorithmic}[1]
\State Set $\cC^l \leftarrow \emptyset, l \in \NN$ and the total level $S=1$
	\For {round $k=1,2,..$}
	        \For {all level $l\in [S]$}
	            \State Denote $\cF^{l} = \{ f | \cL_{\cC^l}(f, \hat{f}^{l}) \le \beta^{l}_{|\cC^l|}  \}$ \label{algline:confidence-set}
	        \EndFor

	    \State Receive the action set $\cA_k$ \label{algline:take-action-1}
	    \State Choose action \\ \qquad $\xb_k\leftarrow \argmax_{\xb \in \cA_k} 
	    \sup_{f \in \bigcap_{l \in [S]} \cF^{l}} f(\xb)$
	    \label{algline:choose-action}
	    \State Receive the reward $R_k$ \label{algline:take-action-2}
	    \State Set level $l=1$ \label{algline:assign-level-1}
	    \While{$w_{\cF^l}(\xb_k) \le 2^{- l}$ and $l \leq  S$} \label{algline:find-level-1}
            \State $l\leftarrow l +1$
        \EndWhile \label{algline:find-level-2}
        \State Add the new element $k$ to the set $\cC^{l}$ 
        and update $\cF^l$ accordingly \label{algline:assign-level-2}
        \State Set the total level $S =\max_{l:|\cC^l|>0} l$ \label{algline:assign-level-3}
	\EndFor
\end{algorithmic}
\end{algorithm}

\subsection{Main Results}

Before presenting the main result for Algorithm~\ref{algorithm:bandit}, we explain how $\beta^l_t$ in ~\eqref{eqn:beta-bandit} is chosen. For each $l \in \NN$, we set $\beta_t^l$ as $\alpha = U_l^{-1}$, where $U_l$ satisfies $
    U_l =
    64 d_K d_E 4^{l} \log U_l/ \delta$.
Later in the proof, we will see that $U_l$ serves as an upper bound on the cardinality of $\cC^l$.

\begin{theorem} \label{thm:bandit}
Suppose $\cF$ satisfies Assumption~\ref{assumption:covering}, and denote $\Delta_k: = \max_{\xb \in \cA_k} f_{\btheta^*}(\xb) -  f_{\btheta^*}(\xb_k)$ and $d_E := \mathrm{dim}_{E}(\cF, \epsilon/2)$. 
Then there exists a constant $c$ such that with probability $1-2\delta$, for all $\epsilon > 0$, Algorithm \ref{algorithm:bandit} satisfies
\begin{align*}
    & 
    {\sum_{k=1}^\infty}
    \ind
    \{ 
    \Delta_k
    >
    \epsilon
    \}
    \le 
    c
    \cdot 
    \frac{d_K d_E}{\epsilon^{2}}
    \log( \frac{d_K d_E}{\epsilon \delta}).
\end{align*}
\end{theorem}


\paragraph{Optimality of the Result}  According to Theorem~3 in~\citet{dann2017unifying}, our result in Theorem~\ref{thm:bandit} can be converted to the same regret bound as \citet{russo2013eluder}, i.e.,  $\tilde{O}(\sqrt{d_K d_E K})$.
Under the linear bandit setting, this result becomes $\tilde{O}(d \sqrt{T})$ because $d_K=d_E = d$  and cannot be improved without additional assumptions. 
Another evidence is that \citet{wagenmaker2022reward} provided a lower bound on fixed-epsilon PAC for linear bandits of $\Omega(d^2/\epsilon^2)$ (Theorem 2). Once again, Theorem~3 in~\citet{dann2017unifying} can convert our result into theirs  since uniform-PAC covers both regret and PAC.
Therefore, this suggests that our result is tight (See Table~\ref{table:bandit} for details).

\paragraph{Computational Efficiency} Generally speaking, the most computationally expensive step is to compute $\arg \sup_{f \in \cF'} f(\xb)$ and $w_{\cF^{l}}(\xb)$, that is to find the optimum function on some given input $\xb$ within the confidence set $\cF'$ (e.g., Line~\ref{algline:choose-action} or Line~\ref{algline:find-level-1}). While these optimization problems can be solved efficiently (or even analytically) for the linear function class, for general function class, the computational efficiency will be more subtly related to the structure of the function class, as well as the optimization algorithm for finding the maximizer.

\section{Uniform-PAC Bounds for Episodic MDPs}
\begin{algorithm*}[h]
	\caption{$\cF$-UPAC-VTR} \label{algorithm:MDP}
	\begin{algorithmic}[1]
    \Require Confidence radius $\beta^l_t(l,t \in \NN)$
    \State Set $\cC^l \leftarrow \emptyset, l \in \NN$ and the total level $S=1$
	\State Denote $\cB^{l} = \{ P \in \cP | \cL_{\cC^l}(P, \hat{P}^{l}) \le \beta^{l}_{|\cC^{l}|}  \}$\label{algline:MDP-confidence-set}
	\For {episode $k=1,2,..$}
        \State Receive the initial state $s_{k,1}$ \label{algline:MDP-planning-1}
	    \State Choose the optimistic model $P_k \leftarrow \argmax_{P \in \bigcap_{l \in  [S]} \cB^{l}} V^{*,P}_{1} (s_{k,1})$ \label{algline:MDP-planning-2}
	    \State Compute value functions $Q_{k,h}$ and $V_{k,h}, h \in [H]$ for $P_k$ according to Equation~\eqref{eq:value-functions}.\label{algline:MDP-value-func}
	    \For {$h=1,2,\dots,H$} \label{algline:MDP-perform-1}
	        \State Choose the current action $a_{k,h} \leftarrow \argmax_{a \in \cA} Q_{k,h}(s_{k,h},a)$ \label{algline:MDP-perform-2}
	        \State Receive the reward and the next state $s_{k,h+1}$  \label{algline:MDP-perform-3}
	        \State Denote $X_{k,h} = (s_{k,h}, a_{k,h}, V_{k,h+1})$ for  $h < H$ \label{algline:MDP-perform-4}
	    \EndFor \label{algline:MDP-perform-5}
	    
	    \For {$h=1,2,\dots,H-1$} \label{algline:MDP-assign-1}
	        \State Set level $l=1$ \label{algline:MDP-assign-2}
            \While{$w_{\cB^l}(X_{k,h})  \le H 2^{- l}$ and $l\leq S$} \label{algline:MDP-assign-3}
                \State $l\leftarrow l +1$ \label{algline:MDP-assign-4}
            \EndWhile \label{algline:MDP-assign-5}
            \State Add the new element $(k,h)$ to the set $\cC^{l}$ (and update $\cB^{l}$ accordingly) \label{algline:MDP-assign-6}
            \State Set $S =\max_{l:|\cC^l|>0}l$ \label{algline:MDP-assign-7}
	    \EndFor \label{algline:MDP-assign-8}
        
	\EndFor
	\end{algorithmic}
\end{algorithm*}

\subsection{Algorithm}
As described in Section~\ref{sec:prelim}, we study the homogeneous episodic MDP, where the unknown, true transition probability $P^*$ lies in a known family $\cP$.

Following~\citet{ayoub2020model}, our results depend on the complexity of a function class $\cF$ associated with $\cP$. Let $\cV$ be the set of optimal value functions under some transition probability in $\cP$, that is $\cV = \{ V^{*, P}_h(\cdot) | h \in [H], P \in \cP \}$. Note that any $V \in \cV$ is positive and bounded by $H$.
Let $\cX = \cS \times \cA \times \cV$, we can see that any triplet $X_{k,h} = (s_{k,h}, a_{k,h}, V_{k,h+1}) \in \cX$. The function class $\cF$ is the collection of functions $f : \cX \rightarrow \RR$ such that
\begin{align*}
    \cF 
    & :=
    \big \{ 
    f_{P}(s,a,V)
    =
    \big \la 
    P(\cdot | s,a)
    ,
    V(\cdot)
    \big \ra 
    \big |
    P \in \cP
    \big \}.
\end{align*}
Note that any $f \in \cF$ is positive and bounded by $H$ because $P(\cdot |s,a)$ is a probability measure and $V(\cdot) \in [0,H]$. We also assume the metric entropy is linearly dependent on $d_K$ as in Assumption~\ref{assumption:covering}.

We present an $\cF$-UPAC-VTR algorithm in Algorithm \ref{algorithm:MDP}.  Similar to Algorithm~\ref{algorithm:bandit}, Algorithm~\ref{algorithm:MDP} will maintain several disjoint sets $\cC^{l}$ , and construct the confidence set (Line~\ref{algline:MDP-confidence-set}) within each set. In particular, we adapt the algorithm design from \citet{ayoub2020model}, and the confidence set $\cB^l$ in Line~\ref{algline:MDP-confidence-set} is defined using:
\begingroup
\allowdisplaybreaks
\begin{align}
    \cL_{\cC^{l}_k}
    (P, \hat{P})
    & :=
    {\sum_{(k,h) \in \cC^{l}_k}}
    \big(
    \big \la 
    P(\cdot | s_{k,h}, a_{k,h})
    -
    \hat{P}(\cdot | s_{k,h}, a_{k,h})
    ,
    V_{k,h+1}(\cdot)
    \big \ra 
    \big)^2,
    \notag \\
    \hat{P}_k^l
    & :=
    \arg \min_{P \in \cP}  
    {\sum_{(k,h) \in \cC^{l}_k}}
    \big(
    \big \la 
    P(\cdot | s_{k,h}, a_{k,h})
    ,
    V_{k,h+1}
    \big \ra 
    -
    V_{k,h+1}(s_{k,h+1})
    \big)^2,
    \notag \\
    \beta^{l}_t
    & := 
    2 H^2 \log
    \big(\cN(\cF, \alpha, \|\cdot\|_{\infty}) / \delta
    \big) 
    +
    {
    2 \alpha t
    \big( 8 H + \sqrt{2 H^2 \log(4t^2/  2^{-l} \delta)}
    \big)}. \label{eqn:beta-MDP}
\end{align}
\endgroup
Based on the confidence sets, Algorithm~\ref{algorithm:MDP} performs the optimistic planning (Line~\ref{algline:MDP-planning-1}-\ref{algline:MDP-perform-5}) at the beginning of each episode $k$. In particular, Algorithm~\ref{algorithm:MDP} first chooses an optimistic model $P_k$ that belongs to the confidence sets of all levels (Line~\ref{algline:MDP-planning-2}). As long as the true transition probability $P^*$ belongs to the confidence sets, the optimistic optimal value function $V^{*,P_k}_1(s_{k,1})$ under $P_k$ will serve as an upper bound on the true value function $V^*_1(s_{k,1})$. 

Then, Algorithm~\ref{algorithm:MDP} computes the optimal value function under the optimistic model via dynamic programming (Line~\ref{algline:MDP-value-func}):
\begin{align}\label{eq:value-functions}
    V_{k,H+1}(s)  &= 0, \notag \\
    Q_{k,h}(s,a) 
    & = 
    r(s,a)
    +
    \big \la 
    P_k(\cdot | s,a)
    ,
    V_{k,h+1}
    (\cdot)
    \big \ra,
    \notag \\
    V_{k,h}(s)
    & = 
    \max_{a \in \cA}
    Q_{k,h}(s,a) .
\end{align}
Based on the value function, the policy $\pi_k$ is implicitly defined as $\pi_{k,h}(s_{k,h}) = \argmax_{a \in \cA} Q_{k,h}(s_{k,h},a)$. The algorithm will follow this policy (Line~\ref{algline:MDP-perform-2}) and receive the corresponding reward and the next state. We denote the triplet of state, action, value function at step $h$ in episode $k$ as $X_{k,h} :=(s_{k,h}, a_{k,h},  V_{k,h+1})$.

After the planning phase, Algorithm~\ref{algorithm:MDP} enters the assigning phase (Line~\ref{algline:MDP-assign-1}-\ref{algline:MDP-assign-7}). At Line~\ref{algline:MDP-assign-3}, we utilize the width $w_{\cB^l}(X_{k,h})$ to determine which level $l$ the index $(k,h)$ should be assigned to. The diameter of a function class $\cB$ on the triplet $X_{k,h}$ is defined as 
\begin{align*}
    w_{\cB}(X_k)
    & :=
    \sup_{P \in \cB} 
    \big \la 
    P(\cdot | s_{k,h}, a_{k,h}),
    V_{k,h+1}(\cdot)
    \big \ra 
    -
    \inf_{P \in \cB}
    \big \la 
    P(\cdot | s_{k,h}, a_{k,h}),
    V_{k,h+1}(\cdot)
    \big \ra. 
\end{align*}
Once all indices are assigned and the confidence sets are updated (Line~\ref{algline:MDP-assign-6}), the algorithm updates the maximum level $S$ and repeats the next episode (Line~\ref{algline:MDP-assign-7}).

\subsection{Main Results}
For each $l \in \NN$, we set $\beta_t^l$ as $\alpha = U_l^{-1}$, where $U_l$ satisfies $
    U_l =
    64 H^2 d_K d_E 4^{l} \log U_l/ \delta$.
In Lemma~\ref{lemma:bounded-set-MDP}, we will see that $U_l$ serves as an upper bound on the cardinality of $\cC^l$.

\begin{theorem} \label{thm:MDP}
Suppose $\cF$ satisfies Assumption~\ref{assumption:covering}, and denote  $d_E := \mathrm{dim}_{E}(\cF, \epsilon/8H)$. 
Then there exists a constant $c$ such that with probability $1-3\delta$, for all $\epsilon > 0$, Algorithm~\ref{algorithm:MDP} satisfies
\begin{align*}
    \sum_{k=1}^\infty
    \ind
    \big\{ 
    V_1^*(s_{k,1})
    -  
    V^{\pi_k}_{k,1}(s_{k,1})
    >
    \epsilon
    \big\}
    & \le 
    c \cdot 
    \frac{H^3 d_K d_{E} \log \big( H^2 d_K d_{E} / (\epsilon \delta) \big )}{\epsilon^2}.
\end{align*}
\end{theorem}
\paragraph{Relation to the Regret Bound in \citet{ayoub2020model}} 
Our result $\tilde{O}(H^3 d_Kd_E \epsilon^{-2})$ can be converted into a regret bound of order $\tilde{O}(\sqrt{H^3 d_Kd_E K})$, where $K$ is the total number of episodes. This matches the regret bound $\tilde{O}(\sqrt{H^3 d_Kd_E K})$ from Theorem~1 in \citet{ayoub2020model} up to logarithmic factors. \\
Additionally, \citet{zhou2020nearly} shows that the regret bound of \citet{ayoub2020model} is nearly tight, off by a factor of $\sqrt{H}$ from optimal, in the setting of linear mixture MDPs.  This suggests that our results under uniform-PAC guarantee can no longer be improved in terms of $d_K$, $d_E$ and $K$ (see Table~\ref{table:mdp} for details). 

\subsection{Potential Extensions to More General RL Classes}
One may ask whether our results can be extended to more general RL classes such as Bellman eluder (BE) dimension \citep{jin2021bellman} and bilinear classes \citep{du2021bilinear}. 
Here, we discuss the possibility of establishing Uniform-PAC guarantees for these RL classes. 

First of all, we would like to clarify that our model-based setting cannot be covered by the BE dimension, because the BE dimension cannot cover linear mixture MDPs (see e.g., Figure 1 by \citet{chen2022general} for a detailed classification), let alone the nonlinear generalization of linear mixture MDPs considered in our work. 

For low BE dimension \citet{jin2021bellman}, the original paper considers two different distribution families, $\cD_{\cF}$ and $\cD_{\Delta}$. $\cD_{\cF}$ generalizes the notion of Bellman rank, while $\cD_{\Delta}$ generalizes the eluder dimension. The ``low BE dimension'' actually means one of the distribution families has a low BE dimension. The multi-level partition scheme can be applied to the GOLF algorithm proposed by \citet{jin2021bellman}, in the sense that we can use the instance-wise Bellman error $|[f_h^k - \cT_h f_h^{k+1}](s_h^k, a_h^k)| \in (2^{-l}, 2^{-l+1}]$ (here $h$ stands for the $h$-th step and $k$ stands for the $k$-th episode) as the criterion to assign step $(k,h)$ to level $l$. By doing so, we can establish the Uniform-PAC guarantee in terms of the BE dimension of the single state-action pair distribution family $\cD_{\Delta}$. On the other hand, it is hard to establish the Uniform-PAC guarantee for the family $\cD_{\cF}$, because we cannot access the expected Bellman error from a single sample. The difference between the expected Bellman error and the instance-wise Bellman error will cause error in the level assignment, which needs to be carefully controlled in order to achieve $\epsilon^{-2}$ sample complexity.  

For bilinear class, \citet{du2021bilinear} assumed the expected Bellman error for some hypothesis $f$ has a bilinear form of $| \la W_h(f) - W_h(f^*), X_h(f) \ra |$. Intuitively, at episode $k$ and step $h$, the quantity $\|X_h(f_k)\|_{\Sigma_{k;h}^{-1}}$ can be used as a criterion for level assignment. However, in most cases, $X_h(f_k)$ is defined as an expectation over the stationary distribution (see e.g., Section 4.3 by \citet{du2021bilinear}), which cannot be computed exactly and has to be estimated using a single sample. Therefore, we again face the challenge of controlling the estimation error of the criterion, which may need a more delicate technique to achieve $\epsilon^{-2}$ PAC bound. 

We leave both extensions as future work.

\section{Proof Overview for Nonlinear Bandits} \label{sec:proof-bandit}
In this section, we show the proof of Theorem~\ref{thm:bandit}, which will illustrate the key idea of achieving uniform-PAC guarantee for bandits and model-based RL with small eluder dimension. 

We will use $\cC^l_{k}$ to denote the level set $\cC^l$ before insert the index $k$ into any level set. $\cB^l_{k}$ and $\hat{P}^l_{k}$ are defined based on $\cC^l_{k}$ accordingly. $l_k$ is the level that the index $k$ is assigned to.

The proof relies on the following two lemmas. The first lemma states that each level set only contains bounded number of indices.
\begin{lemma} \label{lemma:bounded-set}
Suppose we set $\alpha = U_{l}^{-1}$ for $
    \beta^{l}_t= 
    8 \log(\cN(\cF, \alpha, \|\cdot\|_{\infty}) / \delta) +
    2 \alpha t( 8  + \sqrt{8 \log(4t^2/  2^{-l} \delta)})
    $,
where $U_l$ is defined via the equality $
    U_l =
    64 d_K d_E 4^{l} \log U_l/ \delta$.
Then for each level $l > 0$ and each round $k > 0$, the total number of actions in the set $\cC_{k}^l$ is bounded as $| \cC_k^l |
    < 
    U_{l}$,
which also means
\begin{align*}
    | \cC_k^l |
    & < 
    128 d_K d_{E} 4^l \log( 64  d_K d_{E} 4^l / \delta).
\end{align*}
Here, $d_E = \mathrm{dim}_{E}( 2^{-l} )$.
\end{lemma}

The second lemma states the designed confidence set contains the true parameter with high probability.
\begin{lemma} \label{lemma:confidence-set}
Suppose $
    \beta^{l}_t= 
    8 \log(\cN(\cF, \alpha, \|\cdot\|_{\infty}) / \delta) +
    2 \alpha t( 8  + \sqrt{8 \log(4t^2/  2^{-l} \delta)})
    $, and $\cF^{l}_k = \{ f \in \cF | \cL_{\cC^{l}_k}(f, \hat{f}^{l}_{k}) \le \beta^l_{|\cC^l_k|}  \}$
With probability $1-2\delta$, we have for all $k >0$ and $l \in [S_k]$, 
$
    f_{\btheta^*} 
    \in 
    \cF_{k}^{l}.
$
\end{lemma}

Now we proceed to prove the main result.

\begin{proof}[Proof of Theorem~\ref{thm:bandit}]
Denote $\xb_k^{*} = \argmax_{a\in \cA_k} f_{\btheta^*}(\xb) $.
Under the event of Lemma~\ref{lemma:confidence-set}, we have
\begin{align*}
    \max_{a\in \cA_k} f_{\btheta^*}(\xb) 
    -  
    f_{\btheta^*}(\xb_k)
    & 
    =
    f_{\btheta^*}(\xb^*_k) 
    -  
    f_{\btheta^*}(\xb_k)
    \\
    & \le 
    \max_{f \in \cF^{l_k}_k}
    f(\xb^*_k) 
    -
    \min_{f \in \cF^{l_k}_k}
    f(\xb_k) 
    \\
    & \le 
    \max_{f \in \cF^{l_k}_k}
    f(\xb_k) 
    -
    \min_{f \in \cF^{l_k}_k}
    f(\xb_k) \\ 
    & =
    w_{\cF_k^{l_k}}(\xb_k),
\end{align*}
where the first inequality is by Lemma~\ref{lemma:confidence-set}; the second inequality is by the definition of $\xb_k$ that is the optimistic action. The last equality is by the definition of $w_{\cF_k^{l_k}}(\cdot)$. 

Therefore,  by choosing the level $l_0$ such that $2^{-l_0} < \epsilon \le  2^{-(l_0-1)}$, we have
\begin{align*}
    {\sum_{k=1}^\infty
    \ind
    \big\{ 
    \max_{a\in \cA_k} f_{\btheta^*}(\xb) 
    -  
    f_{\btheta^*}(\xb_k)
    >
    \epsilon
    \big\} }
    & \le 
    {\sum_{k=1}^\infty
    \ind
    \big\{ 
    w_{\cF_k^{l_k}}(\xb_k)
    >
    \epsilon
    \big\}}
    \\ 
    & \le 
    {\sum_{k=1}^\infty
    \ind
    \big\{ 
    w_{\cF_k^{l_k}}(\xb_k)
    >
    2^{-l_0}
    \big\}}
    \\
    & \le 
    {\sum_{k=1}^{\infty} 
    \sum_{l=1}^{l_0}
    \ind 
    \big\{ 
    l_k = l
    \big\}
    \le 
    \sum_{l=1}^{l_0}
    U_{l}}
    ,
\end{align*}
where the first line is by the inequality we established above; the second one is by $\epsilon>2^{-l_0}$;  the third inequality is by the design of the level set $\cC^l$ and the last inequality is due to Lemma~\ref{lemma:bounded-set}.

From Lemma~\ref{lemma:bounded-set}, we can further bound $\sum_{l=1}^{l_0}U_{l} $ by
\begin{align*}
    & \sum_{l=1}^{l_0}
    128 d_K \mathrm{dim}_{E}( 2^{-l} ) 4^{l} 
    \log(64 d_K \mathrm{dim}_{E}( 2^{-l} ) 4^{l} / \delta)
    \\
    & \le 
    128 d_K \mathrm{dim}_{E}( \epsilon / 2 )
    \sum_{l=1}^{l_0}
    4^l
    (\log(64 d_K \mathrm{dim}_{E}( \epsilon / 2 )/ \delta) + 2l)
    \\
    & \le 
    128 d_K d_E
    (
    4^{l_0 +1} \log(64 d_K d_E/ \delta)
    +
    4^{l_0  +1} l_0
    )
    \\
    & \le 
    c
    \cdot 
    \frac{d_K d_E}{\epsilon^{2}}
    \log(d_K d_E/ \epsilon \delta),
\end{align*}
where the first line is due to Lemma~\ref{lemma:bounded-set};the second line holds because $2^{-l} \ge 2^{-l_0} \ge \epsilon / 2$ (recall that $d_E = \mathrm{dim}_{E}( \epsilon /2 )$); the third line relies on $\sum_{l=1}^{l_0} 4^{l} < 4^{l_0+1}$ and $\sum_{l=1}^{l_0} l 4^{l} < l_0 4^{l_0+1}$; the last line holds because $2^{-l_0} < \epsilon \le  2^{-(l_0-1)}$, thus $2^{l_0} < 2 \epsilon^{-1}$. $c$ in the last line is some constant. 
\end{proof}

\section{Conclusion}
In this work, we consider online decision making with general functions approximations and proposed two new algorithms for nonlinear bandits and episodic MDPs. With the help of the multi-level partition scheme, our $\mathcal{F}$-UPAC-OFUL algorithm and $\mathcal{F}$-UPAC-VTR algorithm obtain the uniform-PAC guarantee to find the near-optimal policy with the state-of-the-art sample complexity. To the best of our knowledge, these results for the first time show that it is possible to achieve a uniform-PAC guarantee in bandits and RL problems with general function approximation.

\appendix
\section{Proof of Theorem~\ref{thm:MDP}} \label{sec:proof-MDP}

We will use $\cC^l_{k,h}$ to denote the level set $\cC^l$ before assign the tuple $X_{k,h} = (s_{k,h}, a_{k,h}, V_{k,h+1})$ into any level set. $\cB^l_{k,h}$ and $\hat{P}^l_{k,h}$ are defined based on $\cC^l_{k,h}$ accordingly. $l_{k,h}$ denotes the level to which $(k,h)$ is assigned to. Note that $\cF^l_{k,h}$ is defined as:
\begin{align*}
    \cF^l_{k,h}
    & :=
    \bigg \{ 
    f: \cX \rightarrow \RR 
    \bigg | 
    \exists P \in \cB^l_{k,h},
    \text{s.t. } 
    f(s,a,V)
    =
    \int_{s'} 
    P(s'|s,a)
    V(s')
    \bigg \}, 
\end{align*}
where $\cX = \cS \times \cA \times \cB_{\infty}(\cS, H)$ and $\cB_{\infty}(\cS, H)$ is all measurable real-valued functions on $\cS$ that are bounded by $H$.

Similar as in the bandit case, the proof relies on two facts: 1) each level consists finite number of steps; 2) the true transition probability belongs to the confidence sets with high probability.

\begin{lemma} \label{lemma:bounded-set-MDP}
If we set $\alpha = U_{l}^{-1}$ for $
    \beta^{l}_t
    := 
    2 H^2 \log(\cN(\cF, \alpha, \|\cdot\|_{\infty}) / \delta) 
    +
    2 \alpha t( 8 H + \sqrt{2 H^2 \log(4t^2/  2^{-l} \delta)})
    $,
where $U_l$ is defined via the equality $
    U_l
    = 
    64 d_K d_{E} 4^l \log( U_{l} / \delta) $.
Then for each level $l > 0$ and each episode $k > 0$ and time step $h > 0$, the total number of actions in the set $\cC_{k,h}^l$ is bounded as
\begin{align*}
    | \cC_k^l |
    & < 
    U_{l},
\end{align*}
which also means
\begin{align*}
    | \cC_k^l |
    & < 
    128 d_K d_{E} 4^l \log( 64  d_K d_{E} 4^l / \delta).
\end{align*}
Here, $d_E = \mathrm{dim}_E(\cF, H2^{-l})$.
\end{lemma}
\begin{proof}
See Section~\ref{subsec:proof-bounded-set-MDP}.
\end{proof}

The next lemma shows the designed confidence set will contain the true transition probability. The proof of this lemma will follow \citet{ayoub2020model}, while only differs in the definition of the filtration for each level set.
\begin{lemma} \label{lemma:confidence-set-MDP}
[A restatement of Theorem 5 in \citet{ayoub2020model}]
With probability $1-2\delta$, we have for all $k,h >0$ and $l \in S_{k,h}$, 
\begin{align*}
    P^*
    \in 
    \cB_{k,h}^{l}.
\end{align*}
\end{lemma}
\begin{proof}
The proof will be almost same as in \citet{ayoub2020model}, except the martingale design will be multiplied by the indicator $\ind \{ l_{k,h} = l\}$ for some fixed $l$. Since the proof in \citet{ayoub2020model}, also follows that in \citet{russo2013eluder}, we avoid repeating the proof here.
\end{proof}

The next lemma decomposes the pseudo-regret of an episode.
\begin{lemma}[Lemma 4 in \citet{ayoub2020model}] \label{lemma:regret-decomp}
Assuming $P^* \in \cB^{l}$, we have
\begin{align*}
    V_1^*(s_{k,1}) - V^{\pi_k}_{k,1}(s_{k,1})
    & \le 
    \sum_{h=1}^{H-1}
    \big \la 
    P_{k}
    (\cdot| s_{k,h}, a_{k,h})
    -
    P^*
    (\cdot| s_{k,h}, a_{k,h})
    ,
    V_{k,h+1}
    \big \ra 
    +
    \sum_{h=1}^{H-1}
    \xi_{k,h+1},
\end{align*}
where
\begin{align*}
    \xi_{k,h+1} 
    & :=
    \big \la 
    P^*
    (\cdot| s_{k,h}, a_{k,h})
    ,
    V_{k,h+1}( \cdot)
    -
    V^{\pi_k}_{k,h+1}( \cdot)
    \big \ra 
    -
    (V_{k,h+1}( s_{k,h+1})
    -
    V^{\pi_k}_{k,h+1}( s_{k,h+1})).
\end{align*}
Notice that for any given sub-sequence of episodes $k_1, k_2, \dots, k_m$,
$\big[(\xi_{k_i,h+1})_{h \in [H-1]} \big]_{i \in [m]} $ forms a sequence of martingale differences.
\end{lemma}

The following lemma is simply an application of the Azuma's inequality.
\begin{lemma} \label{lemma:martingale-bounded}
For a fixed $l > 0$, denote $k_1, k_2, \dots, k_m$ as the indices of the episodes whose pseudo-regret is above $2^{-l}$ so far.  
With probability $1-\delta / 2^l$, we have 
\begin{align*}
    \sum_{i=1}^{m}
    \sum_{h=1}^{H-1}
    \xi_{k_i,h+1}
    & \le 
    2H \sqrt{2mH \log(2^l \delta^{-1})}.
\end{align*}
By a union bound, we have the above inequality for all $l$ with probability $1-\delta$.
\end{lemma}

\begin{proof}[Proof of Theorem~\ref{thm:MDP}]
First, for any given $\epsilon > 0$, we choose $L$ such that $H 2^{-L} < \epsilon  \le  H 2^{-(L-1)}$. We will bound the number of episodes that has pseudo-regret above $H 2^{-L}$. This scheme will cover all $\epsilon > 0$ since $V_1^*(s_{k,1}) - V^{\pi_{k}}_{k,1}(s_{k,1}) < H$.

At the end of any certain episode, we use $k_1, k_2, \dots, k_m$ to denote the indices of the episodes whose pseudo-regret is above $H2^{-L}$. 

Under the event of Lemma~\ref{lemma:confidence-set-MDP}, 
we have
\begin{align} \label{eqn:regret-inequality}
    m H 2^{-L}
    & \le 
    \sum_{i=1}^{m}
    \big(
    V_1^*(s_{k_i,1}) - V^{\pi_{k_i}}_{k_i,1}(s_{k_i,1})
    \big)
    \notag \\
    & \le 
    \sum_{i=1}^{m}
    \sum_{h=1}^{H-1}
    \big \la 
    P_{k_i}
    (\cdot| s_{k_i,h}, a_{k_i,h})
    -
    P^*
    (\cdot| s_{k_i,h}, a_{k_i,h})
    ,
    V_{k_i,h+1}
    \big \ra 
    +
    \sum_{i=1}^{m}
    \sum_{h=1}^{H-1}
    \xi_{k_i,h+1}
    \notag \\
    & \le 
    \sum_{i=1}^{m}
    \sum_{h=1}^{H-1}
    \bigg[ 
    \sup_{P \in \cB^{l_{k,h}}_{k,h}}
    \big \la 
    P
    (\cdot| s_{k_i,h}, a_{k_i,h})
    ,
    V_{k_i,h+1}
    \big \ra 
    -
    \inf_{P \in \cB^{l_{k,h}}_{k,h}}
    \big \la 
    P
    (\cdot| s_{k_i,h}, a_{k_i,h})
    ,
    V_{k_i,h+1}
    \big \ra
    \bigg]
    +
    \sum_{i=1}^{m}
    \sum_{h=1}^{H-1}
    \xi_{k_i,h+1}
    \notag \\
    & =
    \sum_{i=1}^{m}
    \sum_{h=1}^{H-1}
    w_{\cB^{l_{k,h}}_{k,h}}(X_{k,h})
    +
    \sum_{i=1}^{m}
    \sum_{h=1}^{H-1}
    \xi_{k_i,h+1}
\end{align}
where the second inequality holds due to Lemma~\ref{lemma:regret-decomp}; the third holds under the event of Lemma~\ref{lemma:confidence-set-MDP}. 

Under the event of Lemma~\ref{lemma:martingale-bounded}, we have
\begin{align} \label{eqn:martingale-bound}
    \sum_{i=1}^{m}
    \sum_{h=1}^{H-1}
    \xi_{k_i,h+1}
    & \le 
    2H \sqrt{2mH \log(2^{L} \delta^{-1})}.
\end{align}

Meanwhile, denote $l_0$ such that $H 2^{- l_0} < 2^{-L}/ 2 \le H 2^{- (l_0 - 1)}$ (for simplicity we denote $w_{k,h} = w_{\cB^{l_{k,h}}_{k,h}}(X_{k,h})$)
\begin{align*}
    \sum_{i=1}^{m}
    \sum_{h=1}^{H-1}
    w_{k,h}
    & \le 
    \sum_{i=1}^{m}
    \sum_{h=1}^{H-1}
    \bigg[ 
    \ind \bigg\{
    w_{k,h} > \frac{2^{-L}}{2} 
    \bigg\}
    w_{k,h}
    +
    \frac{2^{-L}}{2}
    \bigg] 
    \notag \\
    & \le 
    \sum_{i=1}^{m}
    \sum_{h=1}^{H-1}
    \ind \bigg\{
    w_{k,h} > \frac{2^{-L}}{2} 
    \bigg\}
    w_{k,h}
    +
    \frac{m H 2^{-L}}{2}
    \notag \\
    & \le 
    \sum_{i=1}^{m}
    \sum_{h=1}^{H-1}
    \ind \{
    l_{k,h} \le  l_0
    \}
    w_{k,h}
    +
    \frac{m H 2^{-L}}{2}
    \notag \\
    & \le 
    \sum_{i=1}^{m}
    \sum_{h=1}^{H-1}
    \ind \{
    l_{k,h} \le  l_0
    \}
    2^{-(l_{k,h} - 1)}
    +
    \frac{m H 2^{-L}}{2},
\end{align*}
where the first inequality is from splitting the case where $w_{k,h} > {2^{-L}}/{2} $ and $w_{k,h} \le {2^{-L}}/{2} $; the third holds because $w_{k,h} > {2^{-L}}/{2} > H 2^{-l_0}$ implies $w_{k,h}$ belongs to the level equal to or lower than $l_0$. 

For some constant $c'$, we have the first term further bounded as:
\begin{align*}
    \sum_{i=1}^{m}
    \sum_{h=1}^{H-1}
    \ind \bigg\{
    l_{k,h} \le  l_0
    \bigg\}
    2^{-(l_{k,h} - 1)}
    & \le 
    \sum_{l=1}^{l_0}
    \sum_{i=1}^{m}
    \sum_{h=1}^{H-1}
    \ind \{
    l_{k,h} =  l
    \}
    2^{-(l - 1)}
    \\
    & =
    \sum_{l=1}^{l_0}
    |\cC_{k,h}^{l}| 2^{-(l - 1)}
    \\
    & \le 
    \sum_{l=1}^{l_0}
    256 d_K \mathrm{dim}_{E}(H 2^{-l}) 2^l \log( 64  d_K d_{E} 4^l / \delta) 
    \\
    & \le 
    \sum_{l=1}^{l_0}
    256 d_K \mathrm{dim}_{E}( 2^{-L} / 4) 2^l \log( 64  d_K d_{E} 4^l / \delta) 
    \\
    & \le 
    256  d_K d_{E} 
    (2^{l_0 + 1} \log( 64  d_K d_{E} / \delta) + 2 l_0 2^{l_0 + 1} )
    \\
    & \le 
    256  d_K d_{E} 
    (8 \cdot 2^{L} H \log( 64  d_K d_{E} / \delta) + 32 \cdot  2^{L} H \log(2^L H)  )
    \\
    & \le 
    c' \cdot 
    H d_K d_{E} 
    2^{L} 
    \log( H d_K d_{E} 2^L / \delta),
\end{align*}
where the third line is due to Lemma~\ref{lemma:bounded-set-MDP}; the fourth line holds because $\mathrm{dim}_E(\cF, \epsilon)$ is decreasing with $\epsilon$ and $H 2^{-l_0} \ge 2^{-L} / 4$; the fifth relies on $\sum_{l=1}^{l_0} 2^{l} < 2^{l_0+1}$ and $\sum_{l=1}^{l_0} l 2^{l} < l_0 2^{l_0+1}$; the sixth is a substitution of $l_0$ by $L$.

Together we have
\begin{align} \label{eqn:bounding-radius}
    \sum_{i=1}^{m}
    \sum_{h=1}^{H-1}
    w_{\cB^{l_{k,h}}_{k,h}}(X_{k,h})
    & \le 
    c' \cdot 
    H d_K d_{E} 
    2^{L} 
    \log( H d_K d_{E} 2^L / \delta)
    +
    \frac{m H 2^{-L}}{2}.
\end{align}
Along with \eqref{eqn:regret-inequality} and \eqref{eqn:martingale-bound}, we have 
\begin{align*}
    m H 2^{-L}
    \le 
    \frac{m H 2^{-L}}{2}
    +
    c' \cdot 
    H d_K d_{E} 
    2^{L} 
    \log( H d_K d_{E} 2^L / \delta)
    +
    2H \sqrt{2mH \log(2^{L} \delta^{-1})},
\end{align*}
Since $m \le A + \sqrt{Bm}$ implies $m \le 2A + 2B$, we have
\begin{align*}
    m 
    & \le 
    4 c'  d_K d_{E} 
    4^{L} 
    \log( H d_K d_{E} 2^L / \delta)
    +
    64 H 4^L \log(2^L / \delta)
    \\
    & \le 
    (4 c'+64) H d_K d_{E} 
    4^{L} 
    \log( H d_K d_{E} 2^L / \delta).
\end{align*}

Finally, since $H 2^{-L} < \epsilon  \le  H 2^{-(L-1)}$, we have $\mathrm{dim}_{E}( 2^{-L} / 4) < \mathrm{dim}_{E}( \epsilon / 8H )$ and 
\begin{align*}
    \sum_{k=1}^\infty
    \ind
    \bigg\{ 
    V_1^*(s_{k,1})
    -  
    V^{\pi_k}_{k,1}(s_{k,1})
    >
    \epsilon
    \bigg\}
    &\le 
    (4 c'+64) H d_K d_{E} 
    4^{L} 
    \log( H d_K d_{E} 2^L / \delta)
    \\
    & \le 
    c \cdot 
    \frac{H^3 d_K d_{E} \log( H^2 d_K d_{E} / \epsilon \delta)}{\epsilon^2},
\end{align*}
for some constant $c$, and $d_E = \mathrm{dim}_{E}( \epsilon / 8H )$. 
This inequality holds for all $\epsilon > 0$ uniformly with probability $1-3 \delta$.
\end{proof}

\section{Proof of Lemmas in Section \ref{sec:proof-bandit}}
\subsection{Proof of Lemma \ref{lemma:bounded-set}} \label{subsec:proof-bounded-set}
The proof relies on Proposition 3 in \citet{russo2013eluder}. A restatement is as follows:
\begin{lemma}[Proposition 3 in \citet{russo2013eluder}] \label{lemma:prop3-russo}
Let $\{ \xb_t \}_{t\in[T]}$ denote a series of actions for some $T > 0$.
If $\{\beta_t\}_{t > 0}$ is a non-decreasing series and $\cF_t = \{ f \in \cF | \cL_{\cC_t}(f, \hat{f}_{t}) \le \beta_{t}  \}$, then 
\begin{align*}
    \sum_{t=1}^{T} \ind \{ w_{\cF_t}(\xb_t) > \epsilon\} 
    & \le 
    \bigg(
    \frac{\beta_{T}}{\epsilon^2}
    +
    1
    \bigg)
    \mathrm{dim}_{E}(\cF, \epsilon).
\end{align*}
\end{lemma}

\begin{proof}[Proof of Lemma~\ref{lemma:bounded-set}]
First, note that 
\begin{align*}
    \beta^{l}_t
    & := 
    2 \log(\cN(\cF, \alpha, \|\cdot\|_{\infty}) / \delta) 
    +
    2 \alpha t( 8  + \sqrt{2 \log(4t^2/  2^{-l} \delta)})
    \\
    & =
    2 \log(\cN(\cF, U_{l}^{-1}, \|\cdot\|_{\infty}) / \delta) 
    +
    2 \frac{t}{U_l} ( 8  + \sqrt{2  \log(4t^2/  2^{-l} \delta)})
    \\
    & \le 
    2  d_K \log( U_{l} / \delta) 
    +
    2 \frac{t}{U_l} ( 8  + \sqrt{2  \log(4t^2/  2^{-l} \delta)}).
\end{align*}

Suppose there are $T$ actions stored in $\cC^l_k$ at round $k$, denote them by the index $(k_1, k_2, \dots, k_T)$.
Notice that each action in level $l$ satisfies $2^{-l} < w_{\cF^{l}_{k_i}}(\xb_{k_i}) \le 2^{-l+1}$ for $i \in [T]$.
Setting $\epsilon = 2^{-l}$, we have by Lemma~\ref{lemma:prop3-russo}
\begin{align*}
    T 
    & \le 
    2\mathrm{dim}_{E}(\cF, 2^{-l} ) \beta^l_T
    4^l
    .
\end{align*}

Now we prove by contradiction. If at some round $k'$, $|\cC_{k'}^l| = U_l$, by setting $T = U_l$, we have
\begin{align*}
    U_l 
    &\le 
    2\mathrm{dim}_{E}(\cF, 2^{-l} ) \beta^l_{U_l}
    4^l
    \\
    & \le 
    2 d_{E} 4^l
    \bigg(
    2  d_K \log( U_{l} / \delta) 
    +
    2 ( 8  + \sqrt{2 \log(4U_l^2/  2^{-l} \delta)})
    \bigg) 
    \\
    & < 
    2 d_{E} 4^l
    \bigg(
    2 d_K \log( U_{l} / \delta) 
    +
    2 ( 8 + \sqrt{2 \log(U_l^2/  2^{-l} \delta)})
    \bigg), 
\end{align*}
and meanwhile
\begin{align} \label{eqn:Ul-definition}
    U_l
    & = 
    64 d_K d_{E} 4^l \log( U_{l} / \delta). 
\end{align}

Combining the two together and rearranging terms, we get
\begin{align*}
    15  \log( U_{l} / \delta)
    & <
    8 + \sqrt{2 (2(\log(4U_l/ \delta) + l + 2)}.
\end{align*}
From \eqref{eqn:Ul-definition}, we see $U_l / \delta > 64 d_K d_{E} 4^l / \delta > 64 \cdot 4^l$. Substituting $U_l / \delta$ to the inequality above leads to a contradiction.

The additional conclusion comes from the fact that $T = A \log T$ implies $T < 2A \log A$ if $A > e$.
\end{proof}

\subsection{Proof of Lemma~\ref{lemma:confidence-set}} \label{subsec:proof-confidence-set}
The proof follows \citet{russo2013eluder}. We include this proof mainly for completeness and to show how the level-partition scheme affects the martingale design. 

Consider random variables $(Z_n | n \in \NN)$ adapted to the filtration $(\cH_n | n \ge 0)$. Assume $\EE[ \exp(\lambda Z_i)]$ is finite for any $\lambda$. Define the conditional mean $\mu_i = \EE[Z_i | \cH_{i-1}]$ and the conditional cumulant generating function of the centered random variable $(Z_i - \mu_i)$ by $\psi_i(\lambda) = \log \EE [\exp \big( \lambda(Z_i - \mu_i) \big) | \cH_{i-1}]$.
Then we have
\begin{lemma}[Lemma 4 in \citet{russo2013eluder}] \label{lemma:lemma4russo}
For all $x \ge 0$ amd $\lambda \ge 0$, 
\begin{align*}
    \PP(\sum_{i=1}^{K} \lambda Z_i \le x + \sum_{i=1}^{K}[ \lambda \mu + \psi_i(\lambda)], \forall K \in \NN) \ge 1 - e^{-x}.
\end{align*}
\end{lemma}

Another lemma regarding the discretization error is:
\begin{lemma}[Lemma 5 in \citet{russo2013eluder}] \label{lemma:lemma5russo}
If $f^{\alpha} $ satisfies $\| f- f^{\alpha} \|_{\infty} \ge \alpha$, then with probability at least $1-\delta$, (denote $t = |\cC^l_k|$ )
\begin{align*}
    \bigg| 
    \frac{1}{2} \cL_{\cC^l_k}(f^{\alpha}, f_{\btheta^*})
    -
    \frac{1}{2} \cL_{\cC^l_k}(f, f_{\btheta^*})
    +
    \sum_{k \in \cC^l_k} \big( f(\xb_k) - R_k \big)^2
    -
    \sum_{k \in \cC^l_k} \big( f^{\alpha}(\xb_k) - R_k \big)^2
    \bigg|
    & \le 
    \alpha t 
    \big[ 
    8
    +
    \sqrt{8 \log (4 t^2 / \delta)}
    \big] , 
\end{align*}
for any given rounds.
\end{lemma}

\begin{proof}[Proof of Lemma~\ref{lemma:confidence-set}]

Now, we consider all those rounds added to the level set $\cC^l$ (with $l$ fixed), and in the end will use a union bound to prove the results for all level.

First we transform our problem to apply the general martingale result. We set $\cH_{k-1}$ to be the $\sigma$-algebra generated by $H_k = (\cA_1, \xb_1, R_1, \dots, \cA_{k-1}, \xb_{k-1}, R_{k-1}, \cA_{k})$ and $\xb_k$. By previous assumptions, $\epsilon_k := R_k - f_{\btheta^*}(\xb_k)$ satisfies $\EE[ \epsilon_k | \cH_{k-1} ] = 0$ and $\EE[ \exp(\lambda \epsilon_k) | \cH_{k-1}] \le \exp ( \lambda^2 / 2)$ since $\epsilon_k$ is $1$-sub-Gaussian.

Define 
\begin{align*}
    Z_k 
    : & =
    \ind \{ 2^{-l} < w_{\cF^l_k}(\xb_k) \le 2^{-(l-1)} \}
    \big[ \big( f_{\btheta^*}(\xb_k) - R_k \big )^2
    -
    \big( f(\xb_k) - R_k \big)^2
    \big]
    \\
    & =
    \ind \{ l_k = l \}
    \big[ 
    - \big( f(\xb_k) - f_{\btheta^*}(\xb_k) \big)^2
    +
    2 \big( f(\xb_k) - f_{\btheta^*}(\xb_k) \big) \epsilon_k
    \big],
\end{align*}
and we have
\begin{align*}
    \mu_k 
    & 
    =
    \EE [Z_k | \cH_{k-1}]
    =
    -\ind \{ l_k = l \}
    \big( f(\xb_k) - f_{\btheta^*}(\xb_k) \big)^2,
    \\
    \psi_{k}(\lambda)
    & = 
    \log 
    \EE 
    \big[ \exp \big \{ 2 \lambda [ f(\xb_k) - f_{\btheta^*}(\xb_k) | \cH_{k-1}] \epsilon_k \big \}   
    \big]
    \le 
    \frac{ \ind \{ l_k = l \} \big( 2 \lambda [ f(\xb_k) - f_{\btheta^*}(\xb_k) ] \big)^2}{ 2}.
\end{align*}
Here, $\ind \{ 2^{-l} < w_{\cF^l_k}(\xb_k) \le 2^{-(l-1)} \} = \ind \{ l_k = l \}$ because it is just a notation change. Note that this indicator function is deterministic on $\cH_{k-1}$. 
Applying Lemma~\ref{lemma:lemma4russo} and set $x = \log \delta^{-1}$ and $\lambda = 1/4$, we have
\begin{align*}
    \PP 
    \bigg( 
    \sum_{k=1}^{K}
    Z_k 
    \le 
    4 \log \delta^{-1} 
    -
    \frac{1}{2}
    \sum_{k=1}^{K}
    \ind \{ l_k = l \}
    \big( f(\xb_k) - f_{\btheta^*}(\xb_k) \big)^2
    , \forall K \in \NN
    \bigg) 
    \ge 1 - \delta .
\end{align*}
Rearranging terms (note $\cL_{\cC^{l}_K}( f, f_{\btheta^*}) = \sum_{k \in \cC^l_K} \big( f(\xb_k) - f_{\btheta^*}(\xb_k) \big)^2$),
\begin{align*}
    \PP 
    \bigg( 
    \sum_{k \in \cC^l_K} \big( f(\xb_k) - R_k \big)^2
    \ge 
    \sum_{k \in \cC^l_K} \big( f_{\btheta^*}(\xb_k) - R_k \big)^2
    -
    4 \log \delta^{-1} 
    +
    \frac{1}{2}
    \cL_{\cC^{l}_K}( f, f_{\btheta^*})
    , \forall K \in \NN
    \bigg) 
    \ge 1 - \delta .
\end{align*}
Let $\cF^{\alpha}$ be a $\alpha$-covering of $\cF$, by a union bound we have with probability $1-\delta$, 
\begin{align*}
    \sum_{k \in \cC^l_K} \big( f^{\alpha}(\xb_k) - R_k \big)^2
    \ge 
    \sum_{k \in \cC^l_K} \big( f_{\btheta^*}(\xb_k) - R_k \big)^2
    -
    4 \log ( |\cF^{\alpha}| / \delta )
    +
    \frac{1}{2}
    \cL_{\cC^{l}_K}( f^{\alpha}, f_{\btheta^*})
    , \forall K \in \NN, \forall f^{\alpha} \in \cF^{\alpha}.
\end{align*}
The inequality above, combined with Lemma~\ref{lemma:lemma5russo}, gives that with probability at least $1-2 \delta$,
\begin{align*}
    \sum_{k \in \cC^l_K} \big( f(\xb_k) - R_k \big)^2
    +
    \alpha t 
    \big[ 
    8
    +
    \sqrt{8 \log (4 t^2 / \delta)}
    \big]
    \ge 
    \sum_{k \in \cC^l_K} \big( f_{\btheta^*}(\xb_k) - R_k \big)^2
    -
    4 \log ( |\cF^{\alpha}| / \delta )
    +
    \frac{1}{2}
    \cL_{\cC^{l}_K}( f, f_{\btheta^*})
    , \forall K \in \NN, \forall f \in \cF.
\end{align*}
By setting $f = \hat{f}_{\cC^{l}_{K}}$, which minimize $\sum_{k \in \cC^l_K} \big( f(\xb_k) - R_k \big)^2$, and rearranging terms, we have with probability $1 - 2\delta_{l}$
\begin{align*}
    \sum_{k \in \cC^l_K}
    \cL_{\cC^{l}_K}( \hat{f}_{\cC^{l}_{K}}, f_{\btheta^*})
    & \le 
    8 \log ( \cN(\cF, \alpha, \|\cdot\|_{\infty}) / \delta_l )
    +
    2 \alpha t 
    \big[ 
    8
    +
    \sqrt{8 \log (4 t^2 / \delta_l)}
    \big].
\end{align*}
Now by setting $\delta_l = \delta 2^{-l}$, and applying union bound, we have with probability $1-2 \delta$,
\begin{align*}
    \cL_{\cC^{l}_K}( \hat{f}_{\cC^{l}_{K}}, f_{\btheta^*})
    & \le 
    \beta^l_{|\cC^l_K|},
    \forall K \in \NN, \forall l \in \NN.
\end{align*}

\end{proof}

\section{Proof of Lemmas in Section \ref{sec:proof-MDP}}
\subsection{Proof of Lemma \ref{lemma:bounded-set-MDP}} \label{subsec:proof-bounded-set-MDP}
\begin{proof}
Note that 
\begin{align*}
    \beta^{l}_t
    & := 
    2 H^2 \log(\cN(\cF, \alpha, \|\cdot\|_{\infty}) / \delta) 
    +
    2 \alpha t( 8 H + \sqrt{2 H^2 \log(4t^2/  2^{-l} \delta)})
    \\
    & =
    2 H^2 \log(\cN(\cF, U_{l}^{-1}, \|\cdot\|_{\infty}) / \delta) 
    +
    2 \frac{t}{U_l} ( 8 H + \sqrt{2 H^2 \log(4t^2/  2^{-l} \delta)})
    \\
    & \le 
    2 H^2 d_K \log( U_{l} / \delta) 
    +
    2 \frac{t}{U_l} ( 8 H + \sqrt{2 H^2 \log(4t^2/  2^{-l} \delta)}).
\end{align*}
Denoting $T=|\cC^l_{k,h}|$ at any fixed time step $(k,h)$, by Lemma~\ref{lemma:prop3-russo},  we know that (set $\epsilon = H 2^{-l}$)
\begin{align*}
    T 
    &\le 
    2\mathrm{dim}_{E}(\cF, H 2^{-l} ) \beta^l_T
    4^l H^{-2}.
\end{align*}

Now, we prove by contradiction. If at some point $T = U_l$, this means
\begin{align*}
    U_l 
    &\le 
    2\mathrm{dim}_{E}(\cF, H 2^{-l} ) \beta^l_{U_l}
    4^l
    \\
    & \le 
    2 d_{E} 4^l H^{-2}
    \bigg(
    2 H^2 d_K \log( U_{l} / \delta) 
    +
    2 ( 8 H + \sqrt{2 H^2 \log(4U_l^2/  2^{-l} \delta)})
    \bigg) 
    \\
    & < 
    2 d_{E} 4^l H^{-2}
    \bigg(
    2 H^2 d_K \log( U_{l} / \delta) 
    +
    2 ( 8 H + \sqrt{2 H^2 \log(4U_l^2/  2^{-l} \delta)})
    \bigg), 
\end{align*}
and meanwhile
\begin{align*}
    U_l
    & = 
    64 d_K d_{E} 4^l \log( U_{l} / \delta). 
\end{align*}
After some rearrangement we see this suggests
\begin{align*}
    15 d_K \log( U_{l} / \delta)
    & \le 
    8
    +
    \sqrt{2(2\log(U_l/\delta) + l + 2)},
\end{align*}
which cannot hold because $U_l/\delta > 64 d_K d_E 4^l / \delta  > 64 \cdot 4^l$.

The additional conclusion comes from the fact that $T = A \log T$ implies $T < 2A \log A$ if $A > e$.
\end{proof}

\bibliographystyle{ims}
\bibliography{reference}

\end{document}